\documentclass[12pt]{article}
\usepackage[utf8]{inputenc}

\title{Knowledge Distillation under Teacher Misspecification:\\
An Order-Parameter Analysis of the Gap between\\ Teacher Mimicry and Task Performance}

\author{
Kazuyuki Hara$^{1}$ and Hideitsu Hino$^{2}$\\
\normalsize{$^{1}$ Nihon University}, \normalsize{$^{2}$ The Institute of Statistical Mathematics}
}
\usepackage{mathtools}
\mathtoolsset{showonlyrefs=true}

\usepackage{makeidx}
\usepackage{amssymb}
\usepackage{amsthm}
\usepackage{amsmath}
\usepackage{amsfonts}
\usepackage{bm}
\usepackage{natbib}
\usepackage{graphicx}
\usepackage{algorithm,algorithmic}
\usepackage{subcaption}

\newcommand{\dmiss}{\delta_{\mathrm{miss}}}   % mismatch strength
\newcommand{\Ktot}{K_{\mathrm{tot}}}          % student capacity
\newcommand{\Ets}{E_{ts}}                     % distillation error
\newcommand{\Etzs}{E_{t_0 s}}                 % true error
\newcommand{\Etzt}{E_{t_0 t}}                 % teacher quality
\newcommand{\nsamp}{n_{\mathrm{samp}}}

\theoremstyle{definition}
\newtheorem{theorem}{Theorem}
\newtheorem*{theorem*}{Theorem}

\newtheorem*{definition*}{Definition}
\theoremstyle{definition}
\newtheorem{lemma}[theorem]{Lemma}
\newtheorem*{lemma*}{Lemma}
\newtheorem{corollary}[theorem]{Corollary}
\newtheorem*{corollary*}{Corollary}
\newtheorem{proposition}[theorem]{Proposition}
\newtheorem*{proposition*}{Proposition}
\newtheorem{remark}[theorem]{Remark}
\newtheorem*{remark*}{Remark}

\newtheorem*{example*}{Example}

\newtheorem*{fact*}{Fact}

\makeindex

\begin{document}

\maketitle

\begin{abstract}
Knowledge distillation trains a small student model to reproduce the outputs of a large teacher model, and its progress is typically monitored through the teacher--student discrepancy. The quantity of ultimate interest, however, is the student's error with respect to the true task. We study the relation between these two objectives in a minimal three-party model, a true teacher (generative model), a teacher, and a student, all soft committee machines, in which the true teacher contains a shared latent factor that the teacher cannot represent, with mismatch strength controlled by a single scalar $\dmiss$. Within an order-parameter description of online distillation, and exploiting closed-form (arcsine-type) expressions for all errors under error-function activations, we prove that the learning dynamics and the distillation error $\Ets$ are exactly invariant to $\dmiss$, whereas the true error $\Etzs$ and the gap $\Delta=\Etzs-\Ets$ are strictly increasing in $\dmiss$, with a rate that is amplified linearly by the complexity $M_0$ of the true teacher. Numerical phase diagrams over the plane spanned by true-teacher complexity and student capacity confirm the predicted deformation: the contours of $\Ets$ do not move while the landscape of $\Etzs$ rises systematically, and a teacher-miss regime, where mimicry succeeds but the task fails, expands with $\dmiss$. The results give a quantitative warning against evaluating distillation solely through teacher-mimicry metrics and identify the gap $\Delta$ as a minimal diagnostic for distinguishing teacher-miss from capacity-limited failure.
\end{abstract}

\section{Introduction}
\label{sec:intro}
The high accuracy of deep learning models is achieved by stacking many layers and using a large number of parameters (degrees of freedom). As a result, training and running deep learning models require enormous computational resources, making their practical deployment difficult.
To address this issue, knowledge distillation has been proposed as a method that reduces computation time while retaining knowledge equivalent to that of a large model. Knowledge distillation~\citep{hinton2015distilling} is a basic technique for transferring the input--output relationship of a large model (the teacher) to a small model (the student), reducing inference cost while maintaining performance. Although it has attracted attention as a key technology for deploying AI systems, theoretical understanding of which large, complex models can be approximated by which learning models has not kept pace, and the scope of applicability of distillation remains unclear.

The learning objective of distillation is typically that the student reproduce the teacher's outputs, and its achievement is measured by the discrepancy between teacher and student, or by the loss against the teacher's soft targets. We collectively refer to metrics of this type as \emph{teacher-output approximation}. What we are ultimately interested in, however, is the student's error on the \emph{true task}, which is to minimize the gap between the true teacher, the generative model, or the unknown target function. We refer to metrics in this sense as the \emph{true error}. Teacher-output approximation and true error tend to agree when the teacher adequately represents the true task, but they are not identical in general. Teacher-output approximation only measures the accuracy of imitating the teacher; when the teacher itself misses part of the true task, there is no guarantee that the true error improves even if the student imitates the teacher perfectly.

The more widely distillation is used in practice, the more this duality becomes an operational risk: an evaluation design that declares success whenever the distillation loss decreases can conflate ``success of imitation'' (teacher-output approximation) with ``success on the task'' (true error). Understanding distillation therefore requires a framework that theoretically separates the two and makes visible the conditions under which they diverge.

\par
The central phenomenon addressed in this paper is the following:
\emph{even though distillation reduces the teacher--student discrepancy (the imitation error, i.e., teacher-output approximation), the error between the true teacher and the student (the true error) does not decrease, or stalls.}

This phenomenon comes in two broad types. The first is \textbf{teacher-miss} (teacher misspecification). When the teacher fails to capture structure possessed by the true teacher, the student's imitation transfers the teacher \emph{including} what it misses: the imitation error decreases, but the error with respect to the task (the true error) does not. The second is \textbf{capacity-limited} failure. Even when the teacher represents the true task well, if the student's capacity (degrees of freedom) is insufficient, the imitation error itself does not decrease sufficiently, and naturally the true error does not improve either.

Importantly, these two failure modes are difficult to distinguish from the same observation (e.g., convergence of the distillation loss). Teacher-miss is a failure in which successful imitation does not translate into task success, while capacity-limited failure is one in which the imitation itself is insufficient. Without this distinction, the appropriate remedy, that is, improving the teacher, enlarging the student, or changing the distillation target, cannot be determined.

We perform this separation by simultaneously tracking the errors among the three parties (the true teacher $t_0$, the teacher $t$, and the student $s$),
\[
\Ets=\frac12\mathbb{E}\bigl[(t-s)^2\bigr],\qquad
\Etzs=\frac12\mathbb{E}\bigl[(t_0-s)^2\bigr],
\]
and in particular using their difference,
\[
\Delta \equiv \Etzs-\Ets,
\]
as a diagnostic that renders teacher-miss visible. A large $\Delta$ indicates a region where imitating the teacher continues to reduce $\Ets$ while the true objective $\Etzs$ does not improve.

The contributions of this paper are as follows. First, we present a framework in which teacher-miss is induced with a minimal modification and its onset is visualized as a phase diagram. Concretely, the true teacher possesses, in addition to latent factors $(U_1,U_2)$, a shared latent factor $U_3$ that the teacher cannot observe, and the strength of this factor is controlled by a single scalar $\dmiss\ge0$. Here $\dmiss=0$ corresponds to no mismatch, and $\dmiss>0$ induces teacher-miss. This design keeps the distillation dynamics (updates targeting the teacher $t$) intact while systematically changing only the evaluation on the true-error side.

Second, we treat distillation as a set of ordinary differential equations (ODEs) for order parameters and analyze the resulting global phase structure. Under a learning rate $\eta$ and continuous time $\alpha$, teacher--student correlations and intra-student correlations (in the symmetric reduction, $r$, $q_d$, $q_o$) follow closed ODEs, in the tradition of the statistical mechanics of online learning~\citep{saad1995exact,biehl1995learning}. This allows systematic statements over the parameter space, where teacher-miss occurs and where capacity limitation dominates, without relying on individual finite-size experiments. Moreover, choosing the error function as the activation renders the distillation error and the true error in closed (arcsine-type) form in the order parameters, and we prove that \emph{the learning dynamics and $\Ets$ are exactly invariant to $\dmiss$, while $\Etzs$ and $\Delta$ are strictly increasing in $\dmiss$, with a rate amplified in proportion to the true-teacher complexity $M_0$} (Section~\ref{subsec:closedform}). Our three-party setting can be viewed as a nonlinear, structured-mismatch counterpart of the linear three-party models studied in the statistical mechanics of learning, such as a student supervised by a moving teacher orbiting a true teacher~\citep{miyoshi2006moving} or by ensembles of imperfect teachers~\citep{miyoshi2006ensemble}.

Third, we present the phase structure through two visualizations: (i) a comparison of the landscapes and contours of $\Ets$ and $\Etzs$ over the $(M_0,\Ktot)$ plane, and (ii) heat maps of the gap $\Delta=\Etzs-\Ets$. Experimentally, increasing $\dmiss$ leaves the contours of $\Ets$ unchanged while the landscape of $\Etzs$ rises systematically and the region of large $\Delta$ expands. This observation shows that the region where improving teacher-output approximation does not guarantee improvement on the true task appears as a phase diagram controlled by the mismatch strength $\dmiss$.

\section{Related Work}
\label{sec:related}

\subsection{Knowledge distillation: foundations and practical issues}
The idea of distillation is old, but its importance was firmly established by \citet{hinton2015distilling}. Early work showing that shallow models can be trained from deep ones includes \citet{ba2014deep}, and methods that transfer knowledge using intermediate representations as hints~\citep{DBLP:journals/corr/RomeroBKCGB14} are important in the context of imposing structural constraints beyond outputs. For a comprehensive survey of knowledge distillation, see \citet{gou2021knowledge}.

It has been observed empirically that distillation fails when the capacity gap between teacher and student is too large, which supports the capacity-limited regime discussed in this paper~\citep{cho2019efficacy}. Teacher assistants have been proposed as a remedy, with an accompanying discussion of the mismatch between teacher complexity and student ability~\citep{mirzadeh2020improved}. Critical examinations showing that distillation does not always yield good generalization~\citep{stanton2021does} provide empirical grounding for our claim that $\Ets$ can decrease while $\Etzs$ does not.

\subsection{Statistical mechanics of online learning}
The statistical-mechanical characterization of learning dynamics is an established stream in the theory of machine learning~\citep{watkin1993statistical,engel2001statistical,goldt2019generalisationdynamicsonlinelearning}. The landmark description of online learning dynamics of multilayer networks by ODEs, which forms the mathematical basis of this paper, is \citet{saad1995exact}; the companion paper analyzes symmetry breaking and the specialization transition~\citep{saad1995on}. General treatments of online gradient learning and its deterministic ODE description are given by \citet{biehl1995learning}, and detailed analyses of the plateau structure and its dependence on the learning rule and metric are given by \citet{rattray1999analysis}.

The framework has been extended well beyond the single teacher--single student configuration. \citet{hara2005ensemble} analyzed online ensemble learning in which linear perceptron students, possibly with teacher- and student-side noise, learn from a common linear teacher, and \citet{miyoshi2005analysis} treated ensembles of nonlinear simple perceptrons within the same online-learning framework. Mutual learning among many students, organized with or without an explicit teacher, has also been analyzed~\citep{hara2009mutual}. These works demonstrate that macroscopic order parameters, namely, teacher--student overlaps and student--student overlaps, suffice to describe the generalization behavior of multi-network systems, which is precisely the reduction we adopt for distillation. Modern updates of this program include the high-dimensional dynamics of generalization error~\citep{advani2020high}, the SGD dynamics of over-parameterized two-layer networks in the teacher--student setup~\citep{goldt2019dynamics}, and computational-to-statistical gap analyses for committee machines~\citep{aubin2018committee}.

\subsection{Model mismatch and three-party teacher--student models}
\label{subsec:related_mismatch}
The setting in which the model class of the learner (or, in our case, of the teacher) cannot represent the data-generating rule has a long history as the study of \emph{unrealizable rules}. Differences between the learning curves of realizable and unrealizable rules were analyzed by \citet{seung1992statistical}; see also \citet{engel2001statistical}. Our study can be positioned as an extension of this line to the three-party relation of true teacher, teacher, and student. Model mismatch has also been studied from the input-distribution side as covariate shift, where training and evaluation distributions differ~\citep{shimodaira2000improving,kimura2022information}; in the present paper the mismatch resides instead in the supervision signal, i.e., in the function class of the teacher that generates the training targets.

Closest to our setting are the three-party models analyzed in the statistical mechanics of online learning. \citet{miyoshi2006moving} treated a system composed of a true teacher, a \emph{moving} teacher orbiting it, and a student, all linear perceptrons with output noise, and computed the three generalization errors: true teacher versus teacher, true teacher versus student, and teacher versus student, which correspond exactly to our $\Etzt$, $\Etzs$, and $\Ets$. \citet{miyoshi2006ensemble} analyzed a student learning from ensembles of imperfect (quasi-optimal) teachers scattered around a true teacher, and \citet{hara2007latent} analyzed mutual learning among students whose generalization ability is measured with respect to a \emph{latent} teacher that takes no part in the learning. A remarkable finding of this line is that a student can generalize better with respect to the true teacher than the imperfect teacher it observes, provided the learning rate is chosen appropriately. Our work differs in two respects. First, the models are nonlinear soft committee machines, so the interplay between capacity ($\Ktot$) and complexity ($M_0$) is nontrivial. Second, and more importantly, the teacher's deficiency is \emph{structural} rather than stochastic: the missing shared factor $U_3$ is absent from the teacher's function class and never enters the supervision signal, so it cannot be averaged out or recovered by any amount of imitation. This is the mechanism that produces the strict, $M_0$-amplified monotonicity of the gap established in Section~\ref{subsec:closedform}.

\subsection{Theoretical analyses of distillation}
Mathematical analyses of distillation outside statistical mechanics have also been developed. \citet{phuong2019towards} gave a theoretical convergence analysis of distillation in the linear case, providing conditions for when distillation succeeds. Distillation has been analyzed from the viewpoint of bias--variance decompositions~\citep{menon2021statistical}, and self-distillation has been shown to act as amplified regularization in Hilbert space~\citep{mobahi2020self}. Risk bounds for distillation in wide neural networks, including the effect of an \emph{imperfect} teacher, are given by \citet{ji2020knowledge}. Compared with these, which are mostly linear or kernel analyses, we treat nonlinear dynamics of soft committee machines; and compared with the mainstream neural-tangent-kernel analyses~\citep{jacot2018neural}, our approach is of the mean-field/thermodynamic-limit type.

Other related threads concern the divergence of evaluation metrics and teacher quality. Large-scale experiments relating the quality of teacher training to distillation performance are reported by \citet{beyer2022knowledge} (``patience and consistency''), and the relation between label smoothing and distillation is discussed by \citet{muller2019does}. Distillation~\citep{hinton2015distilling} and privileged information~\citep{JMLR:v16:vapnik15b} are two techniques that enable machines to learn from other machines, and a unifying view is given by \citet{lopez2015unifying}. Statistical learning theory also provides transfer-risk bounds for distillation~\citep{ji2020knowledge}, and the learning dynamics of over-parameterized networks~\citep{oymak2019overparameterized} are an interesting counterpart to our capacity-limited regime.

Existing work discusses the convergence of $\Ets$ (teacher--student) or generalization bounds, but analytic treatments of the three-way error propagation when the teacher itself deviates from the truth ($\Etzt>0$) are scarce outside the linear models of Section~\ref{subsec:related_mismatch}. Statistical-mechanical methods are well suited to such unrealizable-rule analyses. The novelty of this study is to extend the soft-committee-machine framework of Saad and Solla to the teacher-miss problem in distillation, with a structured latent-factor mismatch controlled by a single scalar, closed-form error evaluation, and visualization as phase diagrams.

\section{Setting and Error Metrics}
\label{sec:setup}

In this section we define the three parties: the true teacher (generative model) $t_0$, the teacher $t$ used for distillation, and the student $s$. Then, we introduce the distillation error, the true error, and their gap as metrics. As an auxiliary device for phase diagrams, we also define good/teacher-miss/capacity-limited regimes.

The input is $\boldsymbol{\xi}\in\mathbb{R}^N$ with $\boldsymbol{\xi}\sim\mathcal{N}(\bm{0},I_N)$, and all models share the activation function
\begin{equation}
g(u)=\mathrm{erf}\!\left(\frac{u}{\sqrt{2}}\right),\qquad
g'(u)=\sqrt{\frac{2}{\pi}}\,e^{-u^2/2}.
\label{eq:activation}
\end{equation}

\subsection{True teacher: latent factors $(U_1,U_2,U_3)$ and complexity $M_0$}
\label{subsec:true_teacher}
The output of the true teacher (generative model) is a soft committee machine (SCM) with $M_0$ hidden units,
\begin{equation}
t_0(\boldsymbol{\xi})
=\frac{1}{\sqrt{M_0}}\sum_{a=1}^{M_0} g\!\left(y^{(0)}_a(\boldsymbol{\xi})\right),
\label{eq:true_teacher_out}
\end{equation}
where the complexity of the true teacher is controlled by $M_0$, which serves as the vertical axis of the phase diagrams.

The units of the true teacher are divided into two groups by the latent factors $(U_1,U_2)$. That is, with $M_0=M_{0,1}+M_{0,2}$, the preactivations are
\begin{align}
y^{(0)}_a &= \sqrt{t_d}\,U_1 + \sqrt{t_{o,1}}\,\varepsilon^{(0)}_a + \sqrt{\dmiss}\,U_3
\qquad (a=1,\dots,M_{0,1}),
\label{eq:true_group1}\\
y^{(0)}_a &= \sqrt{t_d}\,U_2 + \sqrt{t_{o,2}}\,\varepsilon^{(0)}_a + \sqrt{\dmiss}\,U_3
\qquad (a=M_{0,1}+1,\dots,M_{0,1}+M_{0,2}),
\label{eq:true_group2}
\end{align}
where $U_1,U_2,U_3\sim\mathcal{N}(0,1)$ are mutually independent, and $\varepsilon^{(0)}_a\sim\mathcal{N}(0,1)$ are independent unit-specific noises. The parameter $t_d>0$ is the strength of the shared within-group component, and $t_{o,1},t_{o,2}>0$ are the strengths of the group-wise independent noise. In what follows we abbreviate the preactivation variances of group $g\in\{1,2\}$ as
\begin{equation}
\sigma_g^2:=t_d+t_{o,g},\qquad
\sigma_{0,g}^2:=\sigma_g^2+\dmiss ,
\label{eq:sigma_def}
\end{equation}
where the former corresponds to the teacher defined below and the latter to the true teacher.

The scalar $\dmiss\ge0$ is the mismatch strength: for $\dmiss=0$ the true teacher is described by the latent factors $(U_1,U_2)$ alone, while for $\dmiss>0$ only the true teacher contains the additional shared latent factor $U_3$.

\begin{remark}[Realizability of the latent factors]
\label{rem:realizability}
The latent factors and unit-specific noises can be realized as linear projections of the input. Taking mutually orthonormal vectors $\bm{u}_1,\bm{u}_2,\bm{u}_3,\bm{e}^{(0)}_1,\dots\in\mathbb{R}^N$ and setting $U_g=\bm{u}_g^\top\boldsymbol{\xi}$ and $\varepsilon^{(0)}_a=(\bm{e}^{(0)}_a)^\top\boldsymbol{\xi}$, these are independent standard normal variables under $\boldsymbol{\xi}\sim\mathcal{N}(\bm 0,I_N)$, and \eqref{eq:true_group1}--\eqref{eq:true_group2} coincide with the preactivations $y^{(0)}_a=(\bm{B}^{(0)}_a)^\top\boldsymbol{\xi}$ of an ordinary SCM with weight vectors
$\bm{B}^{(0)}_a=\sqrt{t_d}\,\bm{u}_g+\sqrt{t_{o,g}}\,\bm{e}^{(0)}_a+\sqrt{\dmiss}\,\bm{u}_3$.
Consequently, all preactivations appearing below (true teacher, teacher, and student) are jointly Gaussian, and their joint distribution is determined by second moments (order parameters) alone.
\end{remark}

\subsection{Teacher: a mirror model lacking $U_3$}
\label{subsec:teacher}
The teacher $t$ used for distillation is defined as the ``mirror'' of the true teacher with only the additional factor $U_3$ removed. That is, the teacher is an SCM with $M$ units,
\begin{equation}
t(\boldsymbol{\xi})
=\frac{1}{\sqrt{M}}\sum_{m=1}^{M} g\!\left(y_m(\boldsymbol{\xi})\right),
\label{eq:teacher_out}
\end{equation}
with preactivations
\begin{align}
y_m &= \sqrt{t_d}\,U_1 + \sqrt{t_{o,1}}\,\varepsilon_m\qquad (m=1,\dots,M_1),
\label{eq:teacher_group1}\\
y_m &= \sqrt{t_d}\,U_2 + \sqrt{t_{o,2}}\,\varepsilon_m\qquad (m=M_1+1,\dots,M_1+M_2),
\label{eq:teacher_group2}
\end{align}
where the latent factors $(U_1,U_2)$ are \textbf{shared} with the true teacher and the unit-specific noises $\varepsilon_m\sim\mathcal{N}(0,1)$ are drawn independently of $\varepsilon^{(0)}_a$.

In the experiments of this paper we match the structures of the teacher and the true teacher,
\begin{equation}
M=M_0,\qquad M_g=M_{0,g}\quad(g=1,2).
\label{eq:mirror}
\end{equation}
Under this mirror construction, the only difference between teacher and true teacher is the presence or absence of the $U_3$ component, so the effect of teacher-miss is controlled purely by the single scalar $\dmiss$. The generalization in which the teacher capacity $M$ is decoupled from $M_0$ is left as an extension discussed in Section~\ref{subsec:limits_extensions}. Since the teacher does not contain $U_3$ in \eqref{eq:teacher_group1}--\eqref{eq:teacher_group2}, for $\dmiss>0$ a situation systematically arises in which imitating the teacher cannot recover the true factor.

\subsection{Student: capacity ($\Ktot$) and group split}
\label{subsec:student}
The student is an SCM with $\Ktot$ units,
\begin{equation}
s(\boldsymbol{\xi})
=\frac{1}{\sqrt{\Ktot}}\sum_{j=1}^{\Ktot} g\!\left(x_j(\boldsymbol{\xi})\right),
\qquad
x_j(\boldsymbol{\xi})=\boldsymbol{J}_j^\top\boldsymbol{\xi},
\label{eq:student_out}
\end{equation}
where $\boldsymbol{J}_j\in\mathbb{R}^N$ are the learned weight vectors. The horizontal axis of the phase diagrams is $\Ktot$, the capacity of the student.

In the numerical implementation, matching the two-group structure of the teacher, the student is also split into two groups ($\Ktot=K_1+K_2$ with $K_1=\lceil \Ktot/2\rceil$ and $K_2=\lfloor \Ktot/2\rfloor$), with group~1 aligned mainly with the $U_1$ side and group~2 with the $U_2$ side. Under this symmetry we track only representative order parameters (group-wise cross- and self-correlations); see Section~\ref{sec:ode}.

\subsection{Error metrics: $\Ets$, $\Etzs$, and the gap $\Delta$}
\label{subsec:errors}
What distillation minimizes most directly is the output discrepancy between teacher and student. We define the distillation error as
\begin{equation}
\Ets
=\frac12\,\mathbb{E}\!\left[\bigl(t(\boldsymbol{\xi})-s(\boldsymbol{\xi})\bigr)^2\right].
\label{eq:Ets}
\end{equation}
The ultimate objective, in contrast, is the generalization performance of the student with respect to the true teacher (the generative model); the true error is
\begin{equation}
\Etzs
=\frac12\,\mathbb{E}\!\left[\bigl(t_0(\boldsymbol{\xi})-s(\boldsymbol{\xi})\bigr)^2\right].
\label{eq:Et0s}
\end{equation}
Here the expectation $\mathbb{E}[\cdot]$ is over the input $\boldsymbol{\xi}$, which by Remark~\ref{rem:realizability} is equivalent to averaging over the latent factors and noises ($U_1,U_2,U_3,\varepsilon,\varepsilon^{(0)}$). As a measure of teacher quality we also use
\begin{equation}
\Etzt=\frac12\,\mathbb{E}\!\left[\bigl(t_0(\boldsymbol{\xi})-t(\boldsymbol{\xi})\bigr)^2\right].
\label{eq:Et0t}
\end{equation}

The diagnostic quantity of interest in this paper is the gap measuring the discrepancy between successful teacher-output approximation and success on the true task,
\begin{equation}
\Delta
:=\Etzs-\Ets .
\label{eq:gap_def}
\end{equation}
A region with large $\Delta$ corresponds to the canonical picture of teacher-miss: imitation of the teacher (decrease of $\Ets$) is achieved, while the true objective (decrease of $\Etzs$) is not. In the phase diagrams we visualize $\Ets$, $\Etzs$, and $\Delta$ over the $(M_0,\Ktot)$ plane and compare the deformation of the landscapes under the mismatch strength $\dmiss$.

The gap admits the following identity, which decomposes it into teacher quality and the correlation between the teacher's miss and the imitation residual.

\begin{proposition}[Gap decomposition]
\label{prop:gap_decomp}
For any $t_0$, $t$, $s$,
\begin{equation}
\Delta=\Etzs-\Ets
=\Etzt+\mathbb{E}\bigl[(t_0-t)(t-s)\bigr].
\label{eq:gap_decomp}
\end{equation}
\end{proposition}
\begin{proof}
Take expectations of $(t_0-s)^2-(t-s)^2=\bigl((t_0-t)+(t-s)\bigr)^2-(t-s)^2=(t_0-t)^2+2(t_0-t)(t-s)$ and multiply by $1/2$.
\end{proof}

That is, the gap is the teacher quality $\Etzt$ (which grows with $\dmiss$) corrected by the correlation between the teacher's miss $(t_0-t)$ and the imitation residual $(t-s)$. In particular, when imitation is perfect ($s=t$), $\Delta=\Etzt$: the gap that remains after completing imitation is exactly the teacher quality.

\subsection{Regimes: good / teacher-miss / capacity-limited}
\label{subsec:regime}
As an auxiliary device for the phase diagrams, we classify regions by thresholding the errors. We fix two thresholds $\tau_{ts}>0$ and $\tau_{t_0s}>0$ and define three regimes:
\begin{itemize}
\item \textbf{good}: $\Ets\le\tau_{ts}$ and $\Etzs\le\tau_{t_0s}$. Both the distillation error and the true error are small.
\item \textbf{teacher-miss}: $\Ets\le\tau_{ts}$ and $\Etzs>\tau_{t_0s}$. Imitation succeeds but the true error is large.
\item \textbf{capacity-limited}: $\Ets>\tau_{ts}$ and $\Etzs>\tau_{t_0s}$. The imitation error does not decrease (e.g., due to insufficient student capacity) and the true error is also large.
\end{itemize}
(The remaining region, $\Ets>\tau_{ts}$ and $\Etzs\le\tau_{t_0s}$, is treated as ``other'' for convenience.)

Since threshold classifications are definition-dependent, the main conclusions of this paper are stated in terms of the continuous quantities \eqref{eq:Ets}--\eqref{eq:gap_def} ($\Ets$, $\Etzs$, $\Delta$). The regime maps serve only to aid the visualization of boundaries.

\section{Order-Parameter (ODE) Dynamics}
\label{sec:ode}

In this section we describe how, in the high-dimensional limit, the dynamics of distillation (learning that targets the teacher $t$) is described by ordinary differential equations that close in a small number of order parameters. Our implementation uses a symmetric reduction (exchangeable units) in which the student is split into two groups matching the two teacher groups (latent factors $U_1,U_2$). As a result, the order parameters to be updated reduce to three per group: the cross-correlation $r_g$, the diagonal self-correlation $q_{d,g}$, and the off-diagonal self-correlation $q_{o,g}$.

\subsection{Choice of order parameters (symmetric reduction)}
\label{subsec:orderparams}

As defined in Section~\ref{subsec:teacher}, the teacher preactivation in group $g\in\{1,2\}$ is
\begin{equation}
y=\sqrt{t_d}\,U_g+\sqrt{t_{o,g}}\,\varepsilon,
\qquad U_g,\varepsilon\sim\mathcal{N}(0,1)\ \text{independent},
\label{eq:teacher_preact}
\end{equation}
with $U_g$ shared within the group.

For the student, under the assumption that the $K_g$ units in group $g$ are exchangeable (identically distributed with identical correlations), the covariance structure of any unit preactivation $x_i^{(g)}$ in group $g$ is characterized by
\begin{equation}
\mathrm{Var}(x_i^{(g)})=q_{d,g},\qquad
\mathrm{Cov}(x_i^{(g)},x_j^{(g)})=q_{o,g}\ (i\neq j),\qquad
\mathrm{Cov}(y,x_i^{(g)})=r_g ,
\label{eq:cov_structure}
\end{equation}
where $r_g$ is the cross-correlation (alignment) between the teacher (group $g$) and the student (group $g$).

\begin{remark}[Closure ansatz]
\label{rem:closure}
The ODEs of this paper are a reduced dynamics restricted to the order parameters $(r_g,q_{d,g},q_{o,g})$ of each group. That is, (i) overlaps between the student and the teacher's unit-specific noises, $\mathrm{Cov}(x_i^{(g)},\varepsilon_m)$, and (ii) cross-group overlaps, $\mathrm{Cov}(x_i^{(g)},U_{g'})$ and $\mathrm{Cov}(x_i^{(g)},x_j^{(g')})$ for $g\neq g'$, are set identically to zero. In the full Saad--Solla-type system these overlaps can also evolve in time, so our ODE is a minimal model obtained by projecting onto the symmetric block structure. Importantly, the main conclusion that $\dmiss$ does not enter the dynamics (Section~\ref{subsec:mismatch_insertion}) does not depend on this reduction: even in the full order-parameter system, the error signal $A=t-s$ contains no $U_3$, so $\dmiss$ never appears in the learning dynamics.
\end{remark}

In the numerical implementation, a Gaussian representation satisfying the covariance structure \eqref{eq:cov_structure} generates the group-$g$ preactivations as
\begin{equation}
x_i^{(g)}
=a_g\,U_g
+\sqrt{q_{o,g}-a_g^2}\,W_g
+\sqrt{q_{d,g}-q_{o,g}}\,Z_i^{(g)},
\qquad i=1,\dots,K_g ,
\label{eq:student_gaussian_rep}
\end{equation}
where $U_g$ is the latent factor shared with the teacher, $W_g$ is a Gaussian factor shared within the group, and $Z_i^{(g)}$ are unit-specific Gaussian factors, with $W_g,Z_i^{(g)}$ independent standard normal. This representation satisfies
\begin{align*}
&\mathrm{Var}(x_i^{(g)})=a_g^2+(q_{o,g}-a_g^2)+(q_{d,g}-q_{o,g})=q_{d,g},
\\
&
\mathrm{Cov}(x_i^{(g)},x_j^{(g)})=a_g^2+(q_{o,g}-a_g^2)=q_{o,g}.
\end{align*}
Moreover, since the teacher preactivation \eqref{eq:teacher_preact} correlates with $x_i^{(g)}$ only through $U_g$,
\begin{equation}
r_g=\mathrm{Cov}(y,x_i^{(g)})=\sqrt{t_d}\,a_g,
\qquad\text{i.e.}\qquad
a_g=\frac{r_g}{\sqrt{t_d}},
\label{eq:a_reparam}
\end{equation}
which provides the reparameterization used below. For the representation \eqref{eq:student_gaussian_rep} to be valid we need
\begin{equation}
0\leq a_g^2=\frac{r_g^2}{t_d}\leq q_{o,g}\leq q_{d,g}.
\label{eq:validity}
\end{equation}
The left inequality $r_g^2/t_d\leq q_{o,g}$ expresses the geometric fact that alignment with the shared factor $U_g$ necessarily generates within-group correlation among student units ($q_{o,g}\ge a_g^2$); in the numerical integration this constraint is imposed in the direction $q_{o,g}\leftarrow\max(q_{o,g},r_g^2/t_d)$ (Section~\ref{subsec:impl_crn}). For $K_g=1$ the parameter $q_{o,g}$ is undefined and the condition becomes $r_g^2/t_d\leq q_{d,g}$.

The teacher and student outputs, including normalization, are
\begin{equation}
t=\frac{1}{\sqrt{M}}\sum_{m=1}^{M} g(y_m),\qquad
s=\frac{1}{\sqrt{\Ktot}}\sum_{j=1}^{\Ktot} g(x_j).
\label{eq:teacher_student_outputs}
\end{equation}
In what follows we write the error signal appearing in the distillation loss as
\begin{equation}
A \equiv t-s
\label{eq:A_def}
\end{equation}
(to be distinguished later from $B\equiv t_0-s$, which appears in the true error).

\subsection{ODE update equations}
\label{subsec:ode_update}
In distillation the teacher is fixed and only the student is updated online: at each step a new input $\boldsymbol{\xi}$ is observed and a gradient step on the instantaneous loss $\frac12 A^2$ is taken,
\begin{equation}
\boldsymbol{J}_j\leftarrow\boldsymbol{J}_j+\frac{\eta}{N}\,A\,g'(x_j)\,\boldsymbol{\xi},
\label{eq:sgd_rule}
\end{equation}
where $\eta$ is an effective learning rate absorbing the factor stemming from the output normalization $1/\sqrt{\Ktot}$ (if the plain gradient $-\partial_{\boldsymbol{J}_j}\frac12A^2=\frac{1}{\sqrt{\Ktot}}A g'(x_j)\boldsymbol{\xi}$ is used with learning rate $\eta_0$, then $\eta=\eta_0/\sqrt{\Ktot}$). In the high-dimensional limit $N\to\infty$, averaging over one-step updates makes the order parameters evolve smoothly in the continuous time $\alpha$ ($\alpha=\mu/N$, with $\mu$ the step count) and yields closed ODEs. Following our implementation, we adopt the minimal ODE that updates only $(r_g,q_{d,g},q_{o,g})$ per group (Remark~\ref{rem:closure}).

Let $x_1^{(g)}$ be a representative unit of group $g$ and $x_2^{(g)}$ another unit of the same group. Then, following the standard SCM derivation~\citep{saad1995on,biehl1995learning}, the ODEs induced by the update rule \eqref{eq:sgd_rule} are
\begin{align}
\frac{dr_g}{d\alpha}
=&\eta\,\Big\langle A\,g'\!\left(x_1^{(g)}\right)\,y^{(g)}\Big\rangle,
\label{eq:ode_r}\\
\frac{dq_{d,g}}{d\alpha}
=&2\eta\,\Big\langle A\,g'\!\left(x_1^{(g)}\right)\,x_1^{(g)}\Big\rangle
+\eta^2\,\Big\langle A^2\,\bigl(g'\!\left(x_1^{(g)}\right)\bigr)^2\Big\rangle,
\label{eq:ode_qd}\\
\frac{dq_{o,g}}{d\alpha}
=&\eta\,\Big\langle A\Big(g'\!\left(x_1^{(g)}\right)x_2^{(g)}+g'\!\left(x_2^{(g)}\right)x_1^{(g)}\Big)\Big\rangle\\
&+\eta^2\,\Big\langle A^2\,g'\!\left(x_1^{(g)}\right)g'\!\left(x_2^{(g)}\right)\Big\rangle
\qquad (K_g\ge 2),
\label{eq:ode_qo}
\end{align}
where $y^{(g)}$ is the preactivation of a representative teacher unit of group $g$, and the expectation $\langle\cdot\rangle$ is with respect to the Gaussian covariance structure \eqref{eq:cov_structure} of Section~\ref{subsec:orderparams} (and the representation \eqref{eq:student_gaussian_rep}).

The numerical integration is performed by the Euler method with learning rate $\eta$ and step size $d\alpha$, with updates
\begin{align}
r_g &\leftarrow r_g + d\alpha\cdot \eta\,\widehat{F}_{r_g},
\label{eq:euler_r}\\
q_{d,g} &\leftarrow q_{d,g} + d\alpha\cdot \Big(2\eta\,\widehat{F}_{1,g}+\eta^2\,\widehat{F}_{2,g}\Big),
\label{eq:euler_qd}\\
q_{o,g} &\leftarrow q_{o,g} + d\alpha\cdot \Big(\eta\,\widehat{G}_{1,g}+\eta^2\,\widehat{G}_{2,g}\Big)\qquad (K_g\ge 2),
\label{eq:euler_qo}
\end{align}
where $\widehat{F}_{r_g},\widehat{F}_{1,g},\widehat{F}_{2,g},\widehat{G}_{1,g},\widehat{G}_{2,g}$ are Monte Carlo estimators of
\[
\widehat{F}_{r_g}\approx\left\langle A\,g'(x_1^{(g)})\,y^{(g)}\right\rangle,\quad
\widehat{F}_{1,g}\approx\left\langle A\,g'(x_1^{(g)})\,x_1^{(g)}\right\rangle,\quad
\widehat{F}_{2,g}\approx\left\langle A^2\,(g'(x_1^{(g)}))^2\right\rangle,
\]
\[
\widehat{G}_{1,g}\approx\left\langle A\bigl(g'(x_1^{(g)})x_2^{(g)}+g'(x_2^{(g)})x_1^{(g)}\bigr)\right\rangle,\quad
\widehat{G}_{2,g}\approx\left\langle A^2\,g'(x_1^{(g)})g'(x_2^{(g)})\right\rangle,
\]
each estimated from $\nsamp$ samples per step; for $K_g=1$ the update \eqref{eq:euler_qo} is not used. The order parameters are clipped as needed to preserve the validity condition \eqref{eq:validity} (Section~\ref{subsec:impl_crn}).

\subsection{Mismatch $\dmiss$}
\label{subsec:mismatch_insertion}

A key design point of this paper is that the mismatch strength $\dmiss$ does not enter the distillation dynamics itself; it enters only the evaluation against the true teacher.

The true-teacher preactivation is generated by \eqref{eq:true_group1}--\eqref{eq:true_group2} of Section~\ref{subsec:true_teacher}, whose preactivation adds the shared factor $\sqrt{\dmiss}\,U_3$,
\begin{equation}
y^{(0)}=\sqrt{t_d}\,U_g+\sqrt{t_{o,g}}\,\varepsilon^{(0)}+\sqrt{\dmiss}\,U_3 ,
\label{eq:true_preact_mismatch}
\end{equation}
while the teacher $t$ used for distillation follows \eqref{eq:teacher_preact} and contains no $U_3$.

Since the update equations \eqref{eq:ode_r}--\eqref{eq:ode_qo} of distillation are driven solely by the error signal $A=t-s$, the parameter $\dmiss$ does not appear on the right-hand side of the ODEs. Consequently, under identical initial conditions and hyperparameters, the time evolution of $(r_g,q_{d,g},q_{o,g})$ hence the behavior of $\Ets$ is exactly invariant under changes of $\dmiss$ (Theorem~\ref{prop:invariance_monotone}(i)).

The true error, on the other hand, involves
\begin{equation}
B\equiv t_0-s ,
\label{eq:B_def}
\end{equation}
and $\dmiss$ acts directly through
\begin{equation}
\Etzs=\frac12\,\mathbb{E}\!\left[B^2\right].
\label{eq:Et0s_again}
\end{equation}
This separation, identical learning dynamics but systematically degrading true error, is the direct cause of the phase-diagram deformation observed in Section~\ref{sec:results}: the contours of $\Ets$ do not move, the landscape of $\Etzs$ rises, and $\Delta=\Etzs-\Ets$ grows. The next subsection quantifies this separation in closed form.

\subsection{Closed-form error evaluation and monotonicity in $\dmiss$}
\label{subsec:closedform}

For the activation \eqref{eq:activation}, Gaussian averages close in terms of the arcsine function~\citep{saad1995exact}.

\begin{lemma}[Arcsine identity]
\label{lem:arcsin}
Let $(u,v)$ be jointly Gaussian with mean zero, variances $V_u,V_v$, and covariance $C$. Then
\begin{equation}
\mathbb{E}\bigl[g(u)g(v)\bigr]
=\frac{2}{\pi}\arcsin\frac{C}{\sqrt{(1+V_u)(1+V_v)}}
=:\psi(C;V_u,V_v).
\label{eq:psi_def}
\end{equation}
\end{lemma}
\begin{proof}
From $g'(u)=\sqrt{2/\pi}\,e^{-u^2/2}$,
\begin{align*}
\mathbb{E}[g'(u)g'(v)]&=\frac{2}{\pi}\mathbb{E}\bigl[e^{-(u^2+v^2)/2}\bigr]
=\frac{2}{\pi}\det(I+\Sigma)^{-1/2}\\
&=\frac{2}{\pi}\bigl((1+V_u)(1+V_v)-C^2\bigr)^{-1/2},
\end{align*}
where $\Sigma$ is the covariance matrix of $(u,v)$. By Price's theorem~\citep{price1958useful}, $\partial_C\,\mathbb{E}[g(u)g(v)]=\mathbb{E}[g'(u)g'(v)]$; noting that $\mathbb{E}[g(u)g(v)]=\mathbb{E}[g(u)]\mathbb{E}[g(v)]=0$ at $C=0$ and integrating in $C$ yields \eqref{eq:psi_def}.
\end{proof}

From Lemma~\ref{lem:arcsin} and the covariance structure of Section~\ref{sec:setup} (counting shared factors within and across groups), all errors admit closed forms in the order parameters.

\begin{proposition}[Closed-form errors]
\label{prop:closedform}
Under the mirror construction \eqref{eq:mirror}, the second moments are
\begin{align}
\mathbb{E}[t^2]&=\frac{1}{M_0}\sum_{g=1}^{2}\Bigl[M_{0,g}\,\psi(\sigma_g^2;\sigma_g^2,\sigma_g^2)
+M_{0,g}(M_{0,g}-1)\,\psi(t_d;\sigma_g^2,\sigma_g^2)\Bigr],
\label{eq:Et2}\\
\mathbb{E}[t_0^2]&=\frac{1}{M_0}\sum_{g=1}^{2}\Bigl[M_{0,g}\,\psi(\sigma_{0,g}^2;\sigma_{0,g}^2,\sigma_{0,g}^2)\nonumber\\
&\hspace{5.2em}+M_{0,g}(M_{0,g}-1)\,\psi(t_d+\dmiss;\sigma_{0,g}^2,\sigma_{0,g}^2)\Bigr]\nonumber\\
&\quad+\frac{2M_{0,1}M_{0,2}}{M_0}\,\psi(\dmiss;\sigma_{0,1}^2,\sigma_{0,2}^2),
\label{eq:Et02}\\
\mathbb{E}[s^2]&=\frac{1}{\Ktot}\sum_{g=1}^{2}\Bigl[K_g\,\psi(q_{d,g};q_{d,g},q_{d,g})
+K_g(K_g-1)\,\psi(q_{o,g};q_{d,g},q_{d,g})\Bigr],
\label{eq:Es2}\\
\mathbb{E}[ts]&=\frac{1}{\sqrt{M_0\Ktot}}\sum_{g=1}^{2}M_{0,g}K_g\,\psi(r_g;\sigma_g^2,q_{d,g}),\nonumber\\
\mathbb{E}[t_0s]&=\frac{1}{\sqrt{M_0\Ktot}}\sum_{g=1}^{2}M_{0,g}K_g\,\psi(r_g;\sigma_{0,g}^2,q_{d,g}),
\label{eq:Ets_cross}
\end{align}
and
\begin{equation}
\Ets=\tfrac12\mathbb{E}[t^2]-\mathbb{E}[ts]+\tfrac12\mathbb{E}[s^2],\qquad
\Etzs=\tfrac12\mathbb{E}[t_0^2]-\mathbb{E}[t_0s]+\tfrac12\mathbb{E}[s^2].
\label{eq:errors_closed}
\end{equation}
\end{proposition}
\begin{proof}
Recall that $\sigma_g^2$ and $\sigma_{0,g}^2$ are defined in Eq.~\eqref{eq:sigma_def}. Each term is Lemma~\ref{lem:arcsin} applied to a pair of preactivations. By definition, the covariances are: $t_d$ for within-group teacher pairs and $0$ across groups; $t_d+\dmiss$ for within-group true-teacher pairs and $\dmiss$ across groups (sharing $U_3$); $\sqrt{t_d}\,a_g=r_g$ for same-group teacher--student and true-teacher--student pairs (Eq.~\eqref{eq:a_reparam}; $U_3,\varepsilon,\varepsilon^{(0)}$ are uncorrelated with the student) and $0$ across groups; $q_{o,g}$ for within-group student pairs and $0$ across groups.
\end{proof}

The central separation of this paper follows immediately from the closed forms \eqref{eq:Et2}--\eqref{eq:errors_closed}.

\begin{theorem}[Invariance and monotonicity in $\dmiss$]
\label{prop:invariance_monotone}
Fix the order parameters $(r_g,q_{d,g},q_{o,g})_{g=1,2}$ arbitrarily (with $r_g\ge0$). Then:
\begin{itemize}
\item[(i)] The right-hand sides of the ODEs \eqref{eq:ode_r}--\eqref{eq:ode_qo} and $\Ets$ do not depend on $\dmiss$. In particular, under identical initial conditions, the entire trajectory of the order parameters and the time evolution of $\Ets$ are exactly invariant to $\dmiss$.
\item[(ii)] $\Etzs$ is strictly increasing in $\dmiss$. Consequently,
$\dfrac{\partial\Delta}{\partial\dmiss}=\dfrac{\partial\Etzs}{\partial\dmiss}>0$.
\end{itemize}
\end{theorem}
\begin{proof}
(i) The joint distribution of $t$, $s$, and the error signal $A=t-s$ is determined solely by preactivations that contain no $U_3$, so $\dmiss$ appears neither in the expectations in \eqref{eq:ode_r}--\eqref{eq:ode_qo} nor in $\Ets$ of \eqref{eq:errors_closed}.

(ii) The $\dmiss$-dependent terms in \eqref{eq:errors_closed} are $\mathbb{E}[t_0^2]$ and $-\mathbb{E}[t_0s]$. By the monotonicity of $\psi$ in its first (covariance) argument and the fact that
\[
\frac{\sigma_{0,g}^2}{1+\sigma_{0,g}^2},\qquad
\frac{t_d+\dmiss}{1+\sigma_g^2+\dmiss},\qquad
\frac{\dmiss}{\sqrt{(1+\sigma_{0,1}^2)(1+\sigma_{0,2}^2)}}
\]
are all strictly increasing in $\dmiss$ (for the second, the difference between denominator and numerator equals the constant $1+t_{o,g}>0$; for the third, the derivative is $\bigl(2ab+\dmiss(a+b)\bigr)/\bigl(2((a+\dmiss)(b+\dmiss))^{3/2}\bigr)>0$ with $a=1+\sigma_1^2$, $b=1+\sigma_2^2$), the moment $\mathbb{E}[t_0^2]$ is strictly increasing. Furthermore, $\psi(r_g;\sigma_{0,g}^2,q_{d,g})$ is (strictly if $r_g>0$) decreasing in $\dmiss$, so $-\mathbb{E}[t_0s]$ is nondecreasing. Since the diagonal contribution is always strict, $\Etzs$ is strictly increasing. As $\Ets$ is invariant by (i), $\partial_{\dmiss}\Delta=\partial_{\dmiss}\Etzs>0$.
\end{proof}

\begin{corollary}[Amplification by $M_0$]
\label{cor:M0_amplify}
For the symmetric split $M_{0,1}=M_{0,2}=M_0/2$ with $t_{o,1}=t_{o,2}=t_o$ (writing $\sigma_0^2:=t_d+t_o+\dmiss$),
\begin{equation}
\frac{\partial\Delta}{\partial\dmiss}
=\frac{M_0}{4}\left[
\frac{\partial}{\partial\dmiss}\psi(t_d+\dmiss;\sigma_0^2,\sigma_0^2)
+\frac{\partial}{\partial\dmiss}\psi(\dmiss;\sigma_0^2,\sigma_0^2)
\right]
+O(\sqrt{M_0}),
\end{equation}
and the bracket is positive. That is, the rate at which the mismatch increases $\Delta$ is amplified linearly in the complexity $M_0$ of the true teacher.
\end{corollary}
\begin{proof}
In \eqref{eq:Et02}, the coefficient of the within-group pair term is $\frac12\cdot\frac{2\cdot\frac{M_0}{2}(\frac{M_0}{2}-1)}{M_0}=\frac{M_0}{4}-\frac12$, that of the cross-group pair term is $\frac12\cdot\frac{2(M_0/2)^2}{M_0}=\frac{M_0}{4}$, that of the diagonal term is $O(1)$, and the contribution of $\mathbb{E}[t_0s]$ is $O(\sqrt{M_0})$ from the form $M_{0,g}K_g/\sqrt{M_0\Ktot}$.
\end{proof}

\begin{remark}[The gap at $\dmiss=0$ under the mirror construction]
\label{rem:delta0}
Under the mirror construction \eqref{eq:mirror}, at $\dmiss=0$ we have $\sigma_{0,g}^2=\sigma_g^2$, so \eqref{eq:Et02} coincides with \eqref{eq:Et2} and the two expressions in \eqref{eq:Ets_cross} coincide; hence $\Etzs=\Ets$, i.e., $\Delta=0$ holds identically over the whole phase diagram. At the same time $\Etzt>0$ (the unit-specific noises $\varepsilon$ and $\varepsilon^{(0)}$ are independent), but the cross term of Proposition~\ref{prop:gap_decomp} exactly cancels it: $\mathbb{E}[(t_0-t)(t-s)]=-\Etzt$. Note that this equality holds under the closure ansatz (Remark~\ref{rem:closure}); in the full order-parameter system the student can develop finite overlaps with the teacher's unit-specific noises $\varepsilon_m$, so $\Delta$ can take small positive values even at $\dmiss=0$.
\end{remark}

\subsection{Implementation}
\label{subsec:impl_crn}

The expectations on the right-hand sides of \eqref{eq:ode_r}--\eqref{eq:ode_qo} are evaluated by Monte Carlo based on the Gaussian representation \eqref{eq:student_gaussian_rep}, prioritizing implementation flexibility: at each time step, $(U_1,U_2,\varepsilon_m,W_g,Z_i^{(g)})$ are generated from $\nsamp$ standard normal draws, and $A=t-s$ together with the required products (e.g., $A g'(x)y$, $A g'(x)x$, $A^2(g'(x))^2$) are averaged to obtain the RHS. The errors $\Ets,\Etzs,\Delta$, in contrast, are evaluated by the closed forms of Proposition~\ref{prop:closedform}, so no Monte Carlo noise enters the evaluation stage (noise affects only the trajectories of the order parameters).

After each Euler update, the validity condition \eqref{eq:validity} is imposed in each group in the order
\begin{equation}
q_{d,g}\leftarrow\max(q_{d,g},0),\quad
r_g\leftarrow\mathrm{clip}\bigl(r_g;0,\sqrt{t_d\,q_{d,g}}\bigr),\quad
q_{o,g}\leftarrow\mathrm{clip}\bigl(q_{o,g};r_g^2/t_d,\,q_{d,g}\bigr).
\label{eq:clip}
\end{equation}
In particular, the lower clip on $q_{o,g}$ implements the geometric constraint that alignment with the shared factor forces within-group correlation (see the discussion below Eq.~\eqref{eq:validity}).

Since, by Theorem~\ref{prop:invariance_monotone}(i), the learning dynamics does not depend on $\dmiss$, the ODE is integrated \textbf{only once} per grid point $(M_0,\Ktot)$, and the final order parameters are evaluated in closed form for all values of $\dmiss$. This corresponds to exact common random numbers (CRN) for the sweep over $\dmiss$, guaranteeing that the difference,  ``the $\Ets$ contours do not move while $\Etzs$ and $\Delta$ do'', is observed as a structural difference, not through Monte Carlo noise. For reproducibility, random numbers are generated from deterministic seeds determined by the grid indices and a base seed.

\section{Results: Phase Diagrams and the Gap}
\label{sec:results}

In this section we visualize the errors $\Ets,\Etzs$ and the gap $\Delta$ defined in Section~\ref{sec:setup} as phase diagrams over the $(M_0,\Ktot)$ plane, and show their systematic deformation under the mismatch strength $\dmiss$.

\subsection{Experimental settings}
\label{subsec:exp_setting}
The model parameters are $t_d=0.5$ and $t_{o,1}=t_{o,2}=0.5$ (hence $\sigma_g^2=1$); the true-teacher split is $M_{0,1}=\lceil M_0/2\rceil$, $M_{0,2}=\lfloor M_0/2\rfloor$; the teacher follows the mirror construction \eqref{eq:mirror}; and the student split is $K_1=\lceil \Ktot/2\rceil$, $K_2=\lfloor \Ktot/2\rfloor$. The grid of the phase diagrams is $M_0\in\{2,4,6,8,10,12\}$ and $\Ktot\in\{1,2,\dots,8\}$. The ODEs are integrated by the Euler method with effective learning rate $\eta=0.2$ and step size $d\alpha=0.1$ up to $\alpha_{\max}=150$ (we confirmed that the order parameters of all grid points have reached stationarity by this time). The initial values are $r_g=0$, $q_{d,g}=0.25$, $q_{o,g}=0$; the number of Monte Carlo samples for the RHS is $\nsamp=2000$ (with $\nsamp\in\{500,2000,8000\}$ in the sensitivity test). The mismatch strengths are $\dmiss\in\{0,0.05,0.15\}$. The thresholds for the regime classification are $\tau_{ts}=0.15$ and $\tau_{t_0s}=0.18$. Random seeds are determined deterministically from the grid indices.

\subsection{Baseline ($\dmiss=0$): agreement of $\Ets$ and $\Etzs$}
\label{subsec:baseline}
We first confirm the mismatch-free baseline ($\dmiss=0$). Under the mirror construction, Remark~\ref{rem:delta0} gives $\Etzs=\Ets$ and $\Delta=0$ exactly over the whole phase diagram: achieving teacher imitation (decrease of $\Ets$) coincides with reducing the true error (decrease of $\Etzs$).

Figure~\ref{fig:contours_main} (left panel, $\dmiss=0$) overlays the contours of $\Ets$ (solid lines) on the heat map of $\Etzs$; in this panel the two landscapes coincide, so teacher-output approximation translates directly into the true task. The error landscape itself decreases with $\Ktot$ and increases with $M_0$: in the leftmost column ($\Ktot=1$), $\Ets$ increases monotonically with $M_0$, reaching $\Ets\approx0.34$ at $M_0=12$ (capacity-limited), while for large $\Ktot$ it decreases to about $\Ets\approx0.09$.

\begin{figure}[t]
  \centering
  \begin{subfigure}{0.32\linewidth}
    \centering
    \includegraphics[width=\linewidth]{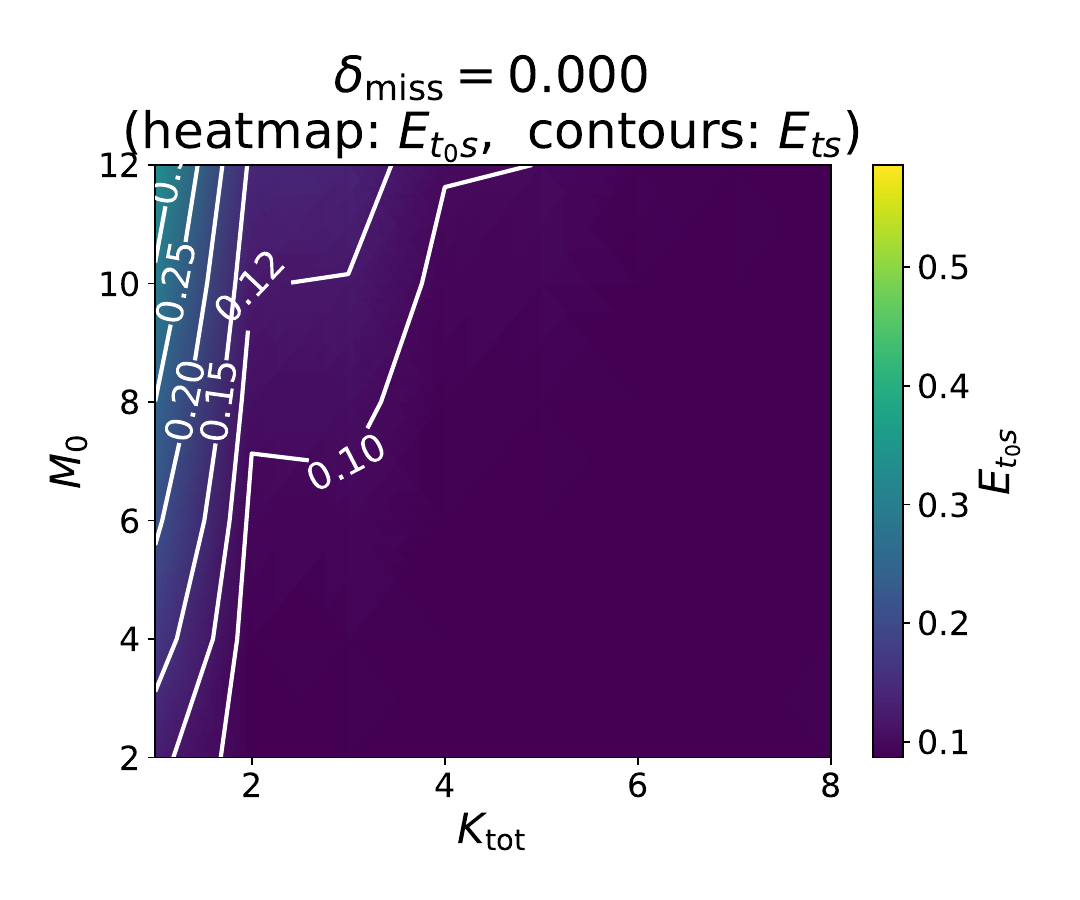}
    \caption{$\dmiss=0.000$}
  \end{subfigure}\hfill
  \begin{subfigure}{0.32\linewidth}
    \centering
    \includegraphics[width=\linewidth]{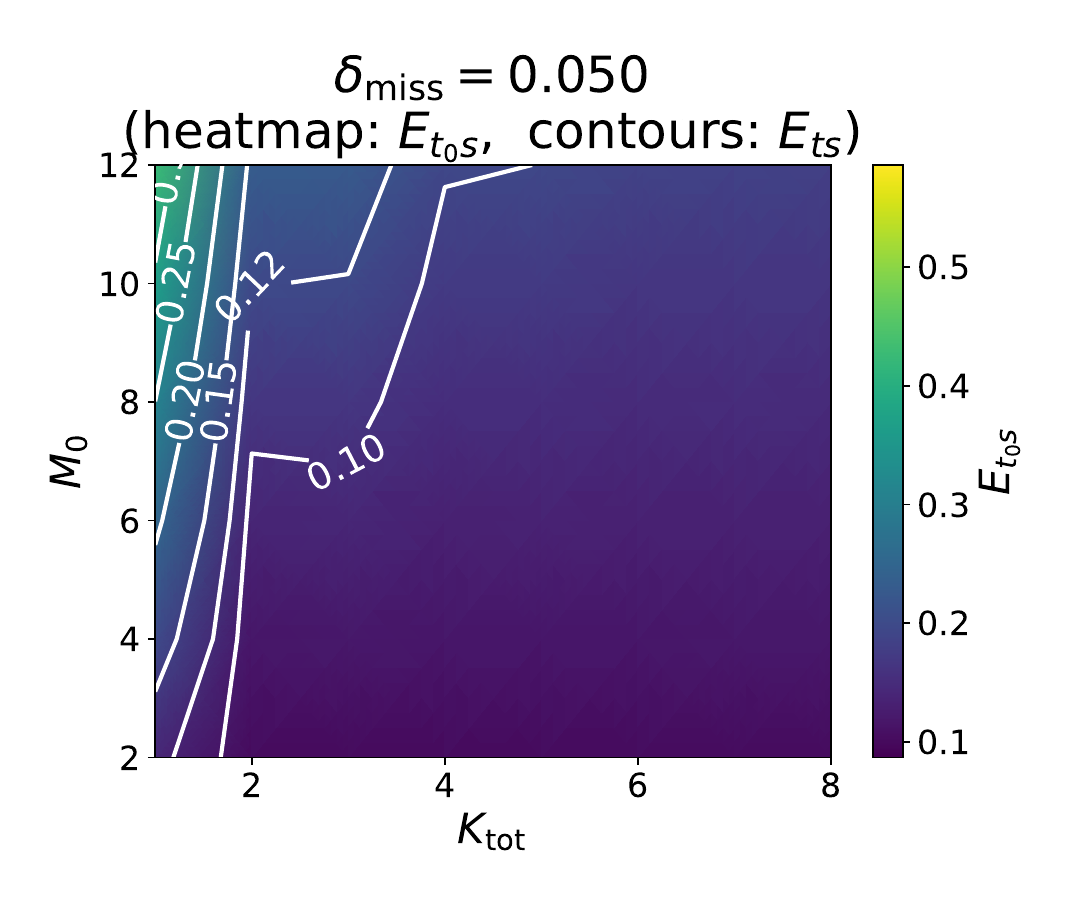}
    \caption{$\dmiss=0.050$}
  \end{subfigure}\hfill
  \begin{subfigure}{0.32\linewidth}
    \centering
    \includegraphics[width=\linewidth]{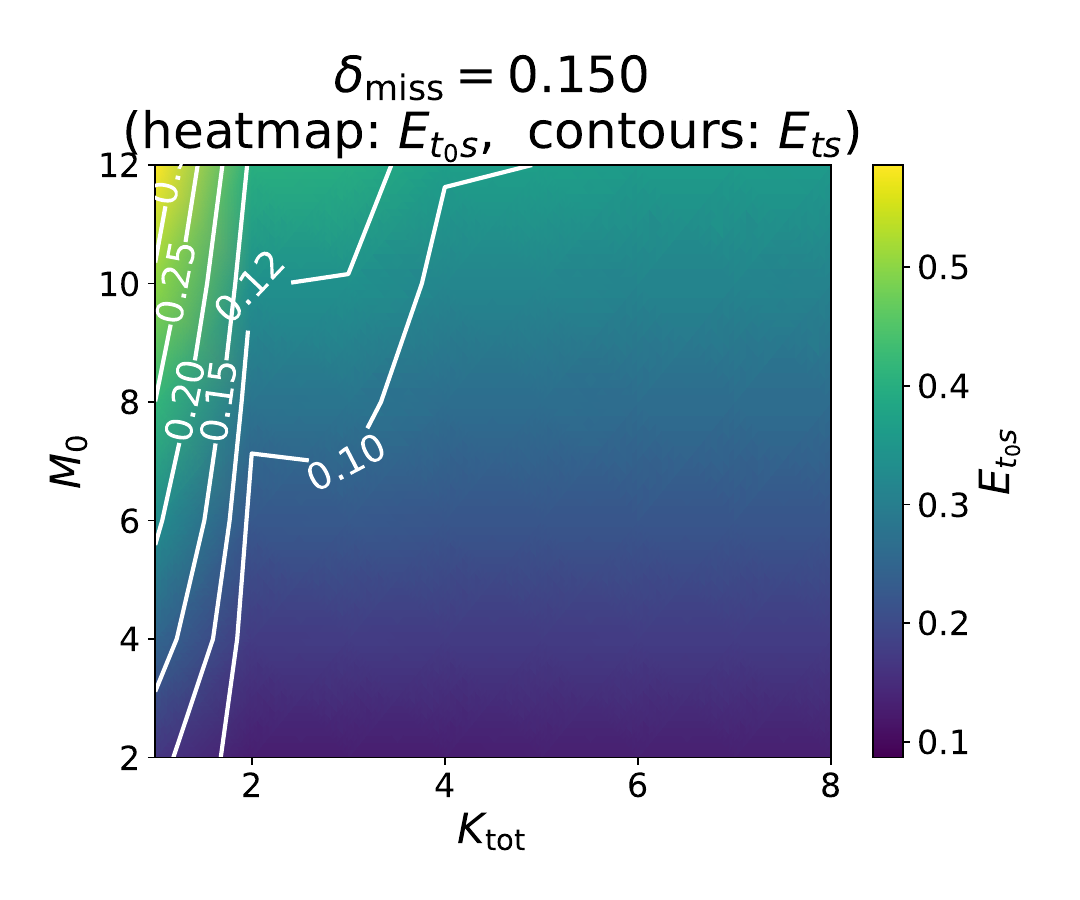}
    \caption{$\dmiss=0.150$}
  \end{subfigure}
  \caption{Separation of the error landscapes under the mismatch strength $\dmiss$. The background heat map shows the true error $\Etzs$ (color scale shared across panels); the solid contours show the distillation error $\Ets$. Increasing $\dmiss$ leaves the $\Ets$ contours unchanged (Theorem~\ref{prop:invariance_monotone}(i)) while the $\Etzs$ landscape rises systematically (Theorem~\ref{prop:invariance_monotone}(ii)). Horizontal axis: student capacity $\Ktot$; vertical axis: true-teacher complexity $M_0$.}
  \label{fig:contours_main}
\end{figure}

\subsection{Increasing mismatch ($\dmiss=0.05,0.15$): separation of contours and landscape}
\label{subsec:mismatch_increase}
Next we introduce mismatch ($\dmiss>0$) and observe the deformation of the phase diagram as its strength increases, using the representative values $\dmiss=0.05$ and $\dmiss=0.15$.

The striking feature of Fig.~\ref{fig:contours_main} (middle and right panels) is that the $\Ets$ contours (solid lines) do not change at all as $\dmiss$ increases. This is exactly Theorem~\ref{prop:invariance_monotone}(i): the distillation ODE is driven only by $A=t-s$, and $\dmiss$ does not enter the learning dynamics (in the implementation the same order-parameter trajectory is reused, so the contours are exactly identical). From the viewpoint of learning that targets the teacher, the latent factor $U_3$ added on the true-teacher side is unobservable, and the difficulty of imitation is unchanged.

The background heat map ($\Etzs$), in contrast, rises systematically as $\dmiss$ increases (Theorem~\ref{prop:invariance_monotone}(ii)). The degradation of $\Etzs$ is more pronounced in regions with large $M_0$ (complex true teachers), consistent with the $M_0$-proportional amplification of Corollary~\ref{cor:M0_amplify}. This landscape change occurs because the term $\sqrt{\dmiss}\,U_3$ added to the true-teacher preactivation contributes to the true error through $B=t_0-s$, visualizing how the objective of distillation ($\Ets$) and the final objective ($\Etzs$) drift apart.

\subsection{Gap: expansion of the region where $\Delta$ grows monotonically}
\label{subsec:gap_map}

\begin{figure}[t]
  \centering
  \includegraphics[width=\linewidth]{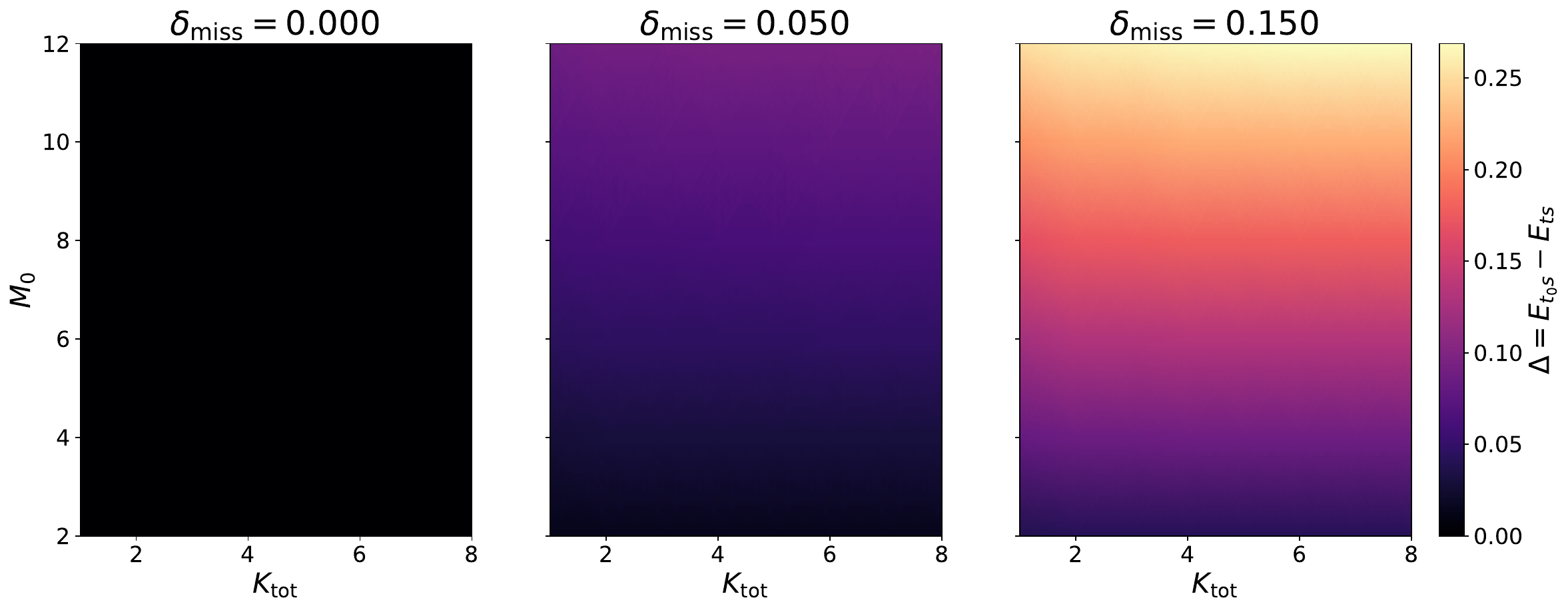}
  \caption{Phase diagram of the gap $\Delta=\Etzs-\Ets$. Each panel corresponds to a different $\dmiss$, with a color scale shared across panels for direct comparison. At $\dmiss=0$ the mirror construction gives $\Delta=0$ identically (Remark~\ref{rem:delta0}); as $\dmiss$ increases, the region with $\Delta>0$ expands from the large-$M_0$ side. Horizontal axis: $\Ktot$; vertical axis: $M_0$.}
  \label{fig:gap_main}
\end{figure}

Figure~\ref{fig:gap_main} shows heat maps of the gap
\begin{equation}
\Delta(M_0,\Ktot;\dmiss)=\Etzs-\Ets .
\end{equation}
The color scale is fixed across panels so that differences due to $\dmiss$ can be compared directly.

Two observations stand out. First, as $\dmiss$ increases, the region where $\Delta$ is positive expands (at $\dmiss=0$, $\Delta\equiv0$). This means the region where the imitation error is small but the true error is large, which is the teacher-miss region, widens. Second, the expansion appears strongly in the direction of increasing $M_0$; in our setting $\Delta\approx0.27$ at $\dmiss=0.15$, $M_0=12$. Intuitively, the more complex the true teacher is and the stronger the factor missed by the teacher, the less teacher imitation translates into true performance (Corollary~\ref{cor:M0_amplify}). We also note that $\Delta$ is nearly flat in the $\Ktot$ direction: the dominant part of $\Delta$ is $\mathbb{E}[t_0^2]-\mathbb{E}[t^2]$, which does not involve the student, while the student-dependent part $\mathbb{E}[ts]-\mathbb{E}[t_0s]$ is comparatively small.

This gap phase diagram quantitatively demonstrates the danger of monitoring only $\Ets$ as the evaluation metric of distillation: a point that looks like a ``success'' region when judged by $\Ets$ alone may be classified as teacher-miss once $\Delta$ is examined. In this sense $\Delta$ is an effective minimal diagnostic for detecting the breakdown of distillation.

\subsection{Regime maps: expansion of the teacher-miss region}
\label{subsec:regime_map}

\begin{figure}[t]
  \centering
  \includegraphics[width=\linewidth]{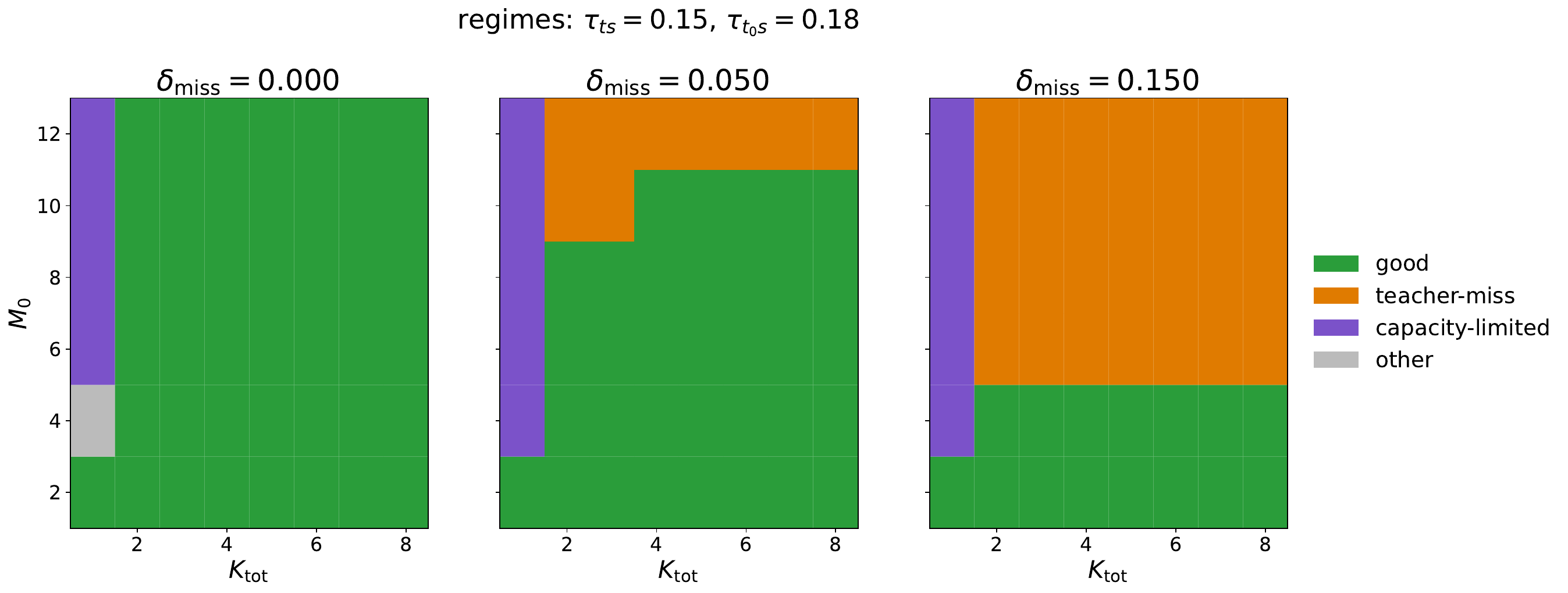}
  \caption{Auxiliary regime maps based on threshold classification with $(\tau_{ts},\tau_{t_0s})=(0.15,0.18)$, dividing the plane into good/teacher-miss/capacity-limited (and other). Because the classification depends on the thresholds, the main conclusions are drawn from the continuous quantities ($\Ets,\Etzs,\Delta$) of Figs.~\ref{fig:contours_main}--\ref{fig:gap_main}.}
  \label{fig:regime_supp}
\end{figure}

Figure~\ref{fig:regime_supp} shows the regime classification (good/teacher-miss/capacity-limited) with thresholds $(\tau_{ts},\tau_{t_0s})$. In this paper the classification is auxiliary; the main conclusions rest on the continuous quantities in Figs.~\ref{fig:contours_main}--\ref{fig:gap_main}. With that caveat, Fig.~\ref{fig:regime_supp} improves the visibility of the boundaries.

The classification shows the expansion of the teacher-miss region with increasing $\dmiss$ as a motion of boundaries. At $\dmiss=0$ the diagram splits only into good and capacity-limited (the small-$\Ktot$ column); as $\dmiss$ increases, points where $\Ets$ is below threshold yet $\Etzs$ exceeds threshold appear from the large-$M_0$ side, widening the region where success judged by the distillation error alone is misleading. Since the boundary locations vary with the choice of thresholds, Fig.~\ref{fig:regime_supp} should be interpreted together with the $\Delta$ maps of Fig.~\ref{fig:gap_main}.

\subsection{Robustness: sensitivity to $\nsamp$}
\label{subsec:robustness}
Because the ODE right-hand sides are evaluated by Monte Carlo, we must confirm that the numerical error induced by the sample size $\nsamp$ does not distort the phase diagrams (the error evaluation itself is in closed form, so $\nsamp$ affects only the order-parameter trajectories). Figure~\ref{fig:nsamp_supp} shows the sensitivity of the regime boundaries and the $\Delta$ contours when $\nsamp\in\{500,2000,8000\}$ (with independently redrawn random seeds). In our setting the global shape of the boundaries is preserved under changes of $\nsamp$, confirming that the conclusion, that is, ``$\Ets$ is invariant to $\dmiss$ while $\Etzs$ and $\Delta$ change'', is not an artifact of numerical noise.

\begin{figure}[t]
  \centering
  \includegraphics[width=\linewidth]{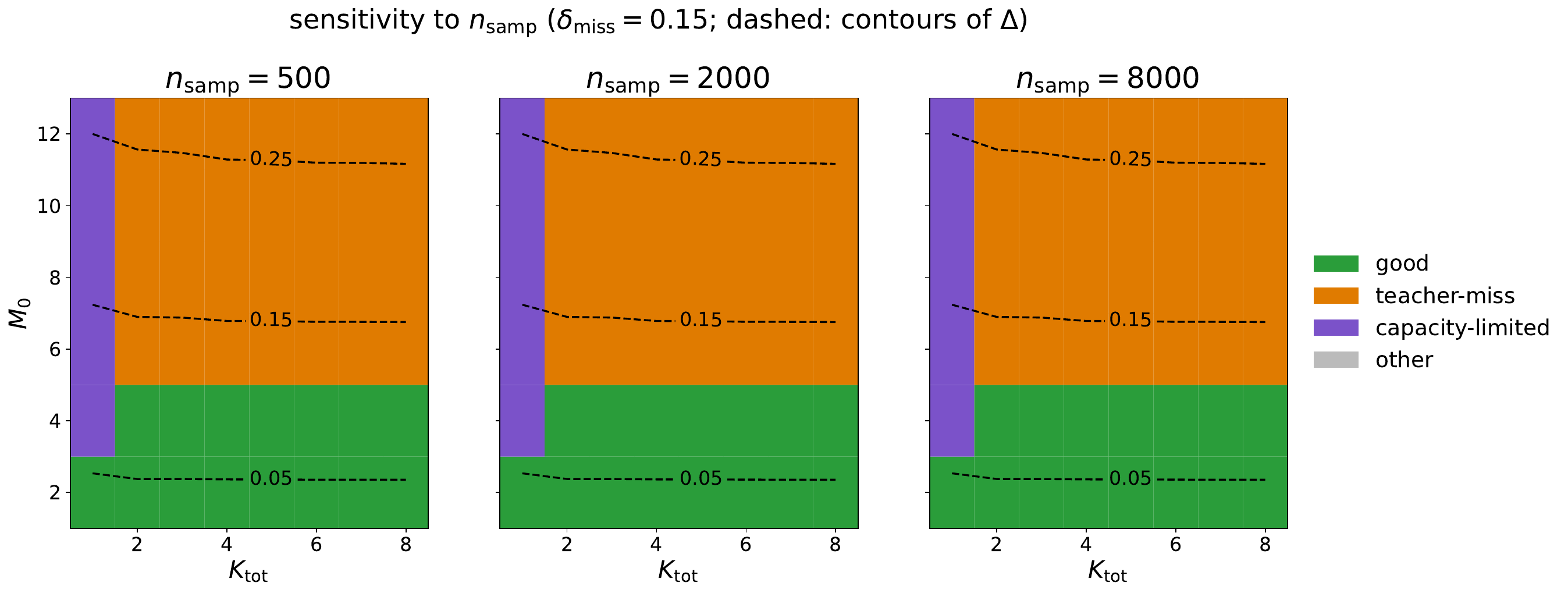}
  \caption{(Auxiliary) sensitivity to $\nsamp$. Regime classification (colors) and contours of $\Delta$ (dashed) at $\dmiss=0.15$, compared for $\nsamp\in\{500,2000,8000\}$ (independent seeds). The global structure of the boundaries is preserved for all $\nsamp$.}
  \label{fig:nsamp_supp}
\end{figure}

\section{Discussion}
\label{sec:discussion}

In this section we organize the phase-diagram deformation observed in Section~\ref{sec:results} from three viewpoints: (i) why it happens, (ii) what it means in practice, and (iii) how far it generalizes. The central phenomenon is the separation in which the distillation error $\Ets$ is invariant to the mismatch strength $\dmiss$ while the true error $\Etzs$ and the gap $\Delta=\Etzs-\Ets$ systematically increase.

\subsection{Why $\Ets$ does not move while $\Etzs$ does}
\label{subsec:why_gap}

This separation is not accidental; it follows necessarily from the structure of the model design and the learning objective (Theorem~\ref{prop:invariance_monotone}). As explained in Section~\ref{subsec:mismatch_insertion}, our mismatch introduces the latent factor $U_3$ only on the true-teacher side; neither the teacher model nor the distillation update rule contains $U_3$. Consequently, the distillation ODE is driven solely by
\[
A=t-s,
\]
and as long as the statistics of $t$ and $s$ are unchanged, varying $\dmiss$ does not change the expected right-hand sides. Hence the order parameters (group-wise $r_g,q_{d,g},q_{o,g}$) at any training time, and the resulting $\Ets$, are invariant to $\dmiss$. This invariance does not rely on the closure ansatz (Remark~\ref{rem:closure}): even in the full order-parameter system, $A$ contains no $U_3$.

The true error, in contrast, is defined through
\[
B=t_0-s,
\]
and the statistics of $t_0$ change with $\sqrt{\dmiss}\,U_3$. This component does not exist on the teacher side, so distillation can neither observe nor optimize the $U_3$-related part. Increasing $\dmiss$ therefore increases the distance between $t_0$ and $s$, and $\Etzs$ deteriorates systematically (Theorem~\ref{prop:invariance_monotone}(ii)). Moreover, the deterioration rate is amplified in proportion to the complexity $M_0$ of the true teacher (Corollary~\ref{cor:M0_amplify}): the missed shared factor $U_3$ rides on every unit of the true teacher, so its contribution accumulates at the order of the number of unit pairs ($\propto M_0^2$) and, even after the $1/\sqrt{M_0}$ normalization, still grows linearly in $M_0$.

Put more plainly, distillation is a procedure that compresses and transfers the teacher's knowledge, so information the teacher does not possess is never transferred. If the teacher misses latent factors of the true teacher, then no matter how accurately the student imitates the teacher, the improvement of the true error is fundamentally limited. Our $\Delta$ phase diagram can be read as a visualization of where, in terms of true complexity $M_0$ and student capacity $\Ktot$, this limitation becomes manifest. It is instructive to contrast this with the linear three-party models of \citet{miyoshi2006moving} and \citet{miyoshi2006ensemble}, in which the imperfection of the observed teacher is stochastic (noise, drift, or ensemble scatter around the truth) and a suitably tuned student can even generalize better than its teacher. In our setting the imperfection is a representational deficit of the teacher's function class; because the missing factor never appears in the supervision signal, no choice of learning rate or capacity allows the student to recover it, and the gap is bounded below by structural terms that grow with $M_0$.

\subsection{Implication: teacher-mimicry metrics alone are insufficient; a gap metric is needed}
\label{subsec:practical}

Using the distillation loss (a teacher-mimicry metric) to monitor progress and declare success is natural. Our results show, however, that this judgment can break down systematically even in a minimal mismatch setting: a decrease of $\Ets$ guarantees that the student is approaching the teacher, but not that it is approaching the true task.

From this viewpoint, at least two elements are important in the evaluation design of distillation pipelines.

First, teacher quality must be assessed outside the distillation loop. When the teacher does not approximate the true teacher well (teacher-miss is large), distillation essentially compresses and transfers this imperfection. In our setting the quality metric is $\Etzt=\frac12\mathbb{E}[(t_0-t)^2]$, and the gap decomposition of Proposition~\ref{prop:gap_decomp},
$\Delta=\Etzt+\mathbb{E}[(t_0-t)(t-s)]$,
shows that in the limit of complete imitation ($s\to t$) we have $\Delta\to\Etzt$: the gap that remains after imitation is exactly the teacher quality.

Second, detecting teacher-miss from quantities observable during distillation requires diagnostics beyond teacher mimicry. The gap $\Delta=\Etzs-\Ets$ proposed here is directly computable in experimental settings where the true teacher is accessible, and it makes teacher-miss visible. In practice the true teacher is often unknown, but there is room to design proxies for $\Delta$ through external evaluations of the teacher or references other than the teacher (different data sets, objectives, or teachers). This concern also interacts with stopping rules: criteria that stop training when a monitored (surrogate) generalization quantity stabilizes, as studied for active learning and Bayesian optimization by \citet{ishibashi2020stopping} and \citet{ishibashi2023stopping}, implicitly assume that the monitored quantity tracks the objective of interest; our analysis exhibits a regime where the natural surrogate ($\Ets$) converges while the target objective ($\Etzs$) has stalled at a level determined by the teacher's miss. At a minimum, our phase diagrams give theoretical support to the claim that teacher-mimicry metrics alone cannot distinguish the failure modes.

In addition, the phase diagrams indicate the direction of remedies. In regions dominated by teacher-miss, increasing the student capacity $\Ktot$ may not improve the true error (corresponding to the near-flatness of $\Delta$ in the $\Ktot$ direction in Fig.~\ref{fig:gap_main}); improving the teacher side (a better teacher, or an enlarged representation space) takes priority. In regions dominated by capacity limitation, the student capacity is the bottleneck even when the teacher quality is sufficient, so increasing $\Ktot$ or improving the architecture acts directly. This distinction reduces wasted trial and error when exploring the design space of distillation.

\subsection{Limitations and extensions}
\label{subsec:limits_extensions}

Our model prioritizes inducing teacher-miss reliably with a minimal modification; the following limitations and extensions are natural.

\paragraph{Multiple hidden factors.}
We expressed the mismatch by a single shared factor $U_3$ with scalar strength $\dmiss$. In reality the missed structure is likely multidimensional. A natural extension replaces $U_3$ by a vector factor (several independent Gaussians) and controls the mismatch strength by a matrix (or its spectrum). The structural fact that latent factors absent from the distillation objective are never learned persists, and the gap phase diagrams should deform in richer ways.

\paragraph{Teacher-side capacity ($M$).}
To control teacher-miss purely through $\dmiss$, we fixed the teacher to the mirror construction \eqref{eq:mirror} ($M=M_0$). More generally, decoupling $M$ from $M_0$ allows analyzing teacher-side capacity shortages beyond the inability to represent $U_3$. This extension is important for studying the interaction between teacher-miss and capacity limitation, and can be presented as three-dimensional phase diagrams over $(M,M_0,\Ktot)$ or as slices in $M$.

\paragraph{Data noise and teacher uncertainty.}
We used Gaussian inputs and SCMs, focusing on model mismatch. Real distillation involves label noise and uncertainty in teacher outputs. These can be incorporated by adding noise to preactivations or observation noise to outputs, with corresponding modifications of the error definitions and the ODE right-hand sides. Importantly, such noise can further blur the relation between teacher-mimicry metrics and the true error, which if anything increases the value of gap-based phase diagrams.

\paragraph{Other activation functions and analytic evaluation.}
We used the error-function activation, and the errors were evaluated by the arcsine closed form of Lemma~\ref{lem:arcsin} (Section~\ref{subsec:impl_crn}). The ODE right-hand sides, in contrast, were evaluated by Monte Carlo for implementation flexibility; replacing them with the known closed forms (Saad--Solla-type multivariate Gaussian integrals) would further improve efficiency and theoretical transparency. The Monte Carlo implementation is robust to modeling changes and readily accommodates extensions (inhomogeneous groups, leakage, multiple factors, activations other than erf), so a practical strategy is to use both: verify with analytic expressions and explore extensions with Monte Carlo.

\section{Conclusion}
\label{sec:conclusion}

We quantified, through a minimal mismatch setting and order-parameter phase diagrams, the situation in which success in teacher-output approximation and success on the true task do not coincide in knowledge distillation. Specifically, the true teacher was extended to include, beyond the teacher's latent space $(U_1,U_2)$, a shared latent factor $U_3$ unobservable to the teacher, with strength controlled by $\dmiss$. Since distillation targets the teacher $t$, the learning dynamics (order-parameter ODEs) is driven only by the error signal $A=t-s$ and $\dmiss$ never enters the right-hand sides, whereas the true error depends on $\dmiss$ directly through $B=t_0-s$.

For this separation structure we established the following, theoretically and experimentally.
\begin{itemize}
\item Under the error-function activation, $\Ets$, $\Etzs$, and $\Delta$ admit arcsine-type closed forms in the order parameters (Lemma~\ref{lem:arcsin}, Proposition~\ref{prop:closedform}).
\item The order-parameter trajectories and $\Ets$ are exactly invariant to $\dmiss$, while $\Etzs$ and $\Delta=\Etzs-\Ets$ are strictly increasing in $\dmiss$, with a rate amplified linearly by the true-teacher complexity $M_0$ (Theorem~\ref{prop:invariance_monotone}, Corollary~\ref{cor:M0_amplify}).
\item The gap decomposes as $\Delta=\Etzt+\mathbb{E}[(t_0-t)(t-s)]$ and coincides with the teacher quality $\Etzt$ in the limit of complete imitation (Proposition~\ref{prop:gap_decomp}).
\item Phase-diagram experiments confirmed that without mismatch ($\dmiss=0$) the landscapes of $\Ets$ and $\Etzs$ coincide (under the mirror construction, $\Delta\equiv0$), and that increasing $\dmiss$ leaves the $\Ets$ contours unchanged while the $\Etzs$ landscape rises systematically and the region of growing $\Delta$ expands from the large-$M_0$ side.
\item The auxiliary threshold-based regime maps also show the expansion of the teacher-miss regime (imitation succeeds but the true error is large) with $\dmiss$; since the classification is threshold-dependent, the main conclusions are based on the continuous quantities ($\Ets,\Etzs,\Delta$).
\item For the Monte Carlo implementation of the RHS expectations, sensitivity tests over $\nsamp$ preserved the global structure of the main boundaries, confirming that the observed phase-diagram deformation is not an artifact of numerical noise.
\end{itemize}

In summary, we have shown as a phase diagram that relying solely on the distillation error (teacher mimicry) as the metric can overlook true performance degradation caused by latent factors the teacher misses. The results support the need, in the evaluation and design of distillation, to distinguish teacher-approximation metrics from true-objective metrics and to introduce a diagnostic (here, $\Delta$) measuring their divergence.

\subsection*{Design rules suggested by the phase diagrams (practical guidance)}
Finally, we summarize the minimal decision rules provided by our phase diagrams. These are practical guidelines based on separating the failure modes, not strict optimality results.
\begin{enumerate}
\item \textbf{Small $\Ets$ but large $\Delta$ (teacher-miss).}
The teacher is likely missing true structure. Increasing the student capacity $\Ktot$ may barely improve the true error, so prioritize teacher-side improvements (enlarging the representation space, a better teacher, or changing the distillation target).
\item \textbf{Large $\Ets$ with small $\Delta$ (capacity-limited).}
The teacher is likely adequate and the student capacity is the bottleneck. Increasing the student capacity or improving its architecture acts directly.
\item \textbf{Both $\Ets$ and $\Etzs$ small (good).}
Distillation is aligned with the objective; further gains are marginal and excessive capacity increases may not be cost-effective.
\end{enumerate}
In practice the true teacher is not directly accessible, but at a minimum one should avoid the shortcut ``improvement of teacher-mimicry metrics $=$ task success,'' and combine external evaluations of teacher quality to detect teacher-miss.

\subsection*{Future work}
Future directions include: (i) generalization to multiple hidden factors (multidimensional mismatch); (ii) phase diagrams over the three-dimensional space $(M,M_0,\Ktot)$ with the teacher capacity $M$ decoupled from $M_0$, to analyze the interaction of teacher-miss and capacity limitation; (iii) extensions to realistic settings with data noise and teacher uncertainty; and (iv) analytic ODE right-hand sides and activations beyond erf, improving theoretical transparency and computational efficiency. Through these, we aim to describe the conditions under which distillation breaks down in greater generality and to consolidate them into evaluation and design guidelines for practice.

\section*{Acknowledgements}
Part of this work is supported by JSPS KAKENHI No. JP25K15267, JP26K23861 and JP26K02989.

%\bibliographystyle{plainnat}
%\bibliography{distil-en}

\end{document}